\documentclass{article}
\title{Imparo is complete by inverse subsumption}
\author{David Toth}

\usepackage[left=2cm,right=2cm,top=2cm,bottom=3cm]{geometry}
\usepackage{amsthm}
\usepackage{amsmath}
\usepackage{amssymb}
\usepackage{listings}
\usepackage{nomencl}
\usepackage{color}
\usepackage{hyperref}

\theoremstyle{plain}
\newtheorem{thm}{Theorem} 

\theoremstyle{plain}

\theoremstyle{plain}

\theoremstyle{definition}
\newtheorem{defn}[thm]{Definition} 

\theoremstyle{definition}

\theoremstyle{definition}

\theoremstyle{definition}

\theoremstyle{definition}

\theoremstyle{definition}

\theoremstyle{plain}

\newcommand{\subsumes}{\succeq}

\begin{document}
   \maketitle

\section*{Abstract}
In Inverse subsumption for complete explanatory induction{\cite{yamamoto2012inverse} Yamamoto et al. investigate which inductive logic programming systems can learn a correct hypothesis $H$ by using the inverse subsumption instead of inverse entailment. We prove that inductive logic programming system Imparo is complete by inverse subsumption for learning a correct definite hypothesis $H$ wrt the definite background theory $B$ and ground atomic examples $E$, by establishing that there exists a connected theory $T$ for $B$ and $E$ such that $H$ subsumes $T$.

\smallskip
\noindent \textbf{Keywords.} Imparo. Inverse subsumption. Inductive Logic Programming.

\section{Introduction}
A task in Inductive Logic Programming (ILP) is given logic theories background knowledge $B$, and examples $E$ to find a logic theory $H$ that explains the examples $E$ from the background knowledge and is consistent with the background knowledge, i.e. $B \land H \models E$ and $B \land H \not\models false$. Such a logic theory $H$ is called a \emph{correct hypothesis} wrt $B$ and $E$ and a system that takes as an input theories $B$ and $E$ and returns a correct hypothesis $H$ is called an ILP system.

ILP systems find a hypothesis $H$ using the principle of the inverse entailment\cite{muggleton1995inverse} for theories $B$, $E$, $H$: $B \land H \models E \iff B \land \neg E \models \neg H$. First they construct an intermediate theory $F$ called a bridge theory satisfying the conditions $B \land \neg E \models F$ and $F \models \neg H$. Then as $H \models \neg F$, they generalize the negation of the bridge theory $F$ with the anti-entailment. However, the operation of the anti-entailment since being less deterministic may be computationally more expensive than the operation of the inverse subsumption (anti-subsumption). Therefore  Yamamoto et al. {\cite{yamamoto2012inverse} investigate how the procedure of the inverse subsumption can be realized in ILP systems in a complete way.

The negation of Imparo's bridge theory is called a connected theory. While Kimber proves that for every hypothesis $H$ there exists a connected theory $T$ such that $H$ entails $T$ ($H \models T$), we prove that for every hypothesis $H$ there exists a connected theory $T$ such that $H$ (theory-)subsumes $T$ ($H \subsumes T$) and hence extend Imparo's procedure for finding a hypothesis from anti-entailment to anti-subsumption preserving its completeness.

\section{Background}
\begin{defn}\cite{yamamoto2012inverse}Let $S$ and $T$ be two clausal theories.
Then, $S$ \emph{theory-subsumes} $T$, denoted by $S \subsumes T$, if for any clause $D \in T$, there is a clause $C \in S$ such
that $C \subsumes D$. The inverse relation of the (theory-)subsumption is called \emph{anti-subsumption}.
\end{defn}

\begin{defn}
(Definition 2.83 in \cite{kimber2012learning}) An \emph{open program} is a triple $\langle B, U, I \rangle$ where $B$ is a program, $U$ is a set of predicates called \emph{undefined} or \emph{abducible}, and $I$ is a set of first-order
axioms. If $B$ is a definite program and $I$ is a set of definite goals, then $P$ is a 
\emph{definite open program}.
\end{defn}

\begin{defn}(Correct hypothesis)
Let $P = \langle B, U, I \rangle$ be a definite open
program, $E$ a logic theory theory called examples, $H$ is an \emph{inductive solution} for $P, E$ iff $B \cup H \models E$ and $B \cup H \cup I \not\models false$.
\end{defn}

\begin{defn}
(Definition 4.1. in \cite{kimber2012learning}
Let $C$ be a program clause $A \leftarrow \{L_1,..., L_n\}$.
The atom $A$ is denoted by $C^+$ and the set
$\{L_1,... ,L_n\}$ is denoted by $C^-$.
\end{defn}

\begin{defn}(Definition 4.2 in \cite{kimber2012learning}). Let $\Sigma$ be a set of $m$ program clauses $\{C_1,... ,C_m\}$.
The set $\{C_1^+,..., C_m^+\}$
is denoted by $\Sigma^+$ and the set $C_1^- \cup ... \cup C_m^-$
is denoted by $\Sigma^-$.
\end{defn}

\begin{defn}(Definition 4.3 in \cite{kimber2012learning})
Let $\langle P = B, U, I \rangle$ be a definite open
program, and let $E$ be a ground atom.
Let $T_1 , . . . , T_n$ be $n$ disjoint sets of ground definite
clauses defining only predicates in $U$. $T = T_1 \cup ... \cup T_n$
is an \emph{$n$-layered Connected Theory}
for $P$ and $E$ if and only if the following conditions are satisfied:
\begin{itemize}
\item $B \models T^-_n$,
\item $B \cup T^+_n \cup ... \cup T^+_{i+1} \models T_i^-$, for all $i (1 \le i < n)$,
\item $B \cup T^+_n \cup ... \cup T^+_1 \models E$, and
\item $B \cup T \cup I$ is consistent.
\end{itemize}
\end{defn}

\begin{defn}
(Definition 4.4 in \cite{kimber2012learning}).
Let $P = \langle B, U, I \rangle$ be a definite open program, and
let $E$ be a ground atom. A \emph{Connected Theory} for $P$ and $E$ is an $n$-layered Connected Theory for $P$ and $E$, for some $n \ge 1$.
\end{defn}

\begin{defn}\cite{kimber2012learning} Let $P=\langle B, U, I \rangle$ be an open definite program, let $E$ be a ground atom. A set $H$ of definite clauses is derivable from $P$ and $E$ by
\emph{Connected Theory Generalisation}, denoted $P, E \vdash_{CTG} H$, iff there is a $T$ such
that $T$ is a Connected Theory for $P$ and $E$, and $H \models T$ , and $B \cup H \cup I$ is consistent.
\end{defn}

\begin{thm}\label{implicationByGroundClauses}
(Implication by Ground Clauses \cite{nienhuys1997foundations}). Let $\Sigma$ be a non-empty set of clauses,
and $C$ be a ground clause. Then $\Sigma \models C$ if and only if there is a finite set $\Sigma_g$ of ground
instances of clauses from $\Sigma$, such that $\Sigma_g \models C$.
\end{thm}

\begin{thm}\label{completeness_ctg}
\emph{Completeness of connected theory generalization}(Theorem4.6 in \cite{kimber2012learning})
Let $\langle B, U, I \rangle$ be a definite open program,
let $H$ be a definite program, and let $e$ be an atom.
If $H$ is an inductive solution for $P$ and
$E = \{e\}$, then $H$ is derivable from $P$ and $E$ by connected theory generalisation.
\end{thm}
\begin{proof}\cite{kimber2012learning}
The full proof is in Kimber's PhD thesis \cite{kimber2012learning}.
Since $H$ is a correct hypothesis for $P$ and $E$,
then $B \cup H \models E$ by definition.
Therefore, by \ref{implicationByGroundClauses}, there is a finite set $S$ of ground instances of clauses in $B \cup H$,
such that $S \models E$. Let $T = S \cap ground(H)$.
Since $T \subseteq S$, then $T$ is ground and finite, and
since $T \subseteq ground(H)$ then $H \models T$. 
Then Kimber proves that $T$ is a connected theory for $P$ and $E$.
\end{proof}

\section{Imparo's extension\cite{toth2014classification}}
We define a derivability of the hypothesis by the inverse subsumption.

\begin{defn}
Let $P=\langle B, U, I \rangle$ be an open definite program, $H$ be a correct hypothesis wrt $P$ and a ground example $E$, then $H$ is derivable by
\emph{connected theory inverse subsumption}
iff there exists a connected theory $T$ for $P$ and $E$ such that $H \subsumes T$.
We denote the statement by $P, E \vdash_{CTIS} H$.
\end{defn}

The result of this paper is:

\begin{thm}
\emph{Completeness of connected theory inverse subsumption}.
\label{completeness_ctis}
Let $\langle B, U, I \rangle$ be a definite open program,
let $H$ be a definite program, and let $e$ be an atom.
If $H$ is an inductive solution for $P$ and
$E = \{e\}$, then $H$ is derivable from $P$ and $E$ by connected theory inverse subsumption.
\end{thm}
\begin{proof}
Construct a connected theory $T=S \cap ground(H)$ for $P$ and $E$ as in the proof of \ref{completeness_ctg}.
Then $H \subsumes ground(H) \subsumes S \cap ground(H) = T$,
hence $H \subsumes T$ by transitivity as required.
\end{proof}

\section*{Acknowledgements}
We thank Dr. Krysia Broda, Dr. Timothy Kimber, Prof. Yoshitaka Yamamoto for checking the contents of drafts of this paper.

\addcontentsline{toc}{chapter}{Bibliography}
\bibliography{bibliography}

\begin{thebibliography}{NCDW97}

\bibitem[Kim12]{kimber2012learning}
Timothy Kimber.
\newblock {\em Learning definite and normal logic programs by induction on
  failure}.
\newblock PhD thesis, Imperial College London, 2012.

\bibitem[Mug95]{muggleton1995inverse}
Stephen Muggleton.
\newblock Inverse entailment and progol.
\newblock {\em New generation computing}, 13(3-4):245--286, 1995.

\bibitem[NCDW97]{nienhuys1997foundations}
Shan-Hwei Nienhuys-Cheng and Ronald De~Wolf.
\newblock {\em Foundations of inductive logic programming}, volume 1228.
\newblock Springer, 1997.

\bibitem[Tot14]{toth2014classification}
David Toth.
\newblock Classification of inductive logic programming systems.
\newblock Master's thesis, Imperial College London, 2014.

\bibitem[YII12]{yamamoto2012inverse}
Yoshitaka Yamamoto, Katsumi Inoue, and Koji Iwanuma.
\newblock Inverse subsumption for complete explanatory induction.
\newblock {\em Machine learning}, 86(1):115--139, 2012.

\end{thebibliography}
\bibliographystyle{alpha}

\end{document}